%% file: example_paper.tex
\documentclass{article}

\usepackage{microtype}
\usepackage{graphicx}
\usepackage{subcaption}
\usepackage{booktabs} 

\usepackage{hyperref}

\usepackage[preprint]{icml2026}

\usepackage{amsmath}
\usepackage{amssymb}
\usepackage{mathtools}
\usepackage{amsthm}

\usepackage[capitalize,noabbrev]{cleveref}

\theoremstyle{plain}

\newtheorem{proposition}{Proposition}
\newtheorem{lemma}{Lemma}
\newtheorem{corollary}{Corollary}
\theoremstyle{definition}
\newtheorem{definition}{Definition}

\theoremstyle{remark}
\newtheorem{remark}{Remark}

\usepackage[textsize=tiny]{todonotes}

\input{preamble}

\icmltitlerunning{Submission and Formatting Instructions for ICML 2026}

\begin{document}

\twocolumn[
    \icmltitle{The Inlet Rank Collapse in Implicit Neural Representations: \\Diagnosis and Unified Remedy}



  \icmlsetsymbol{equal}{*}

  \begin{icmlauthorlist}
    \icmlauthor{Jianqiao Zheng}{yyy}
    \icmlauthor{Hemanth Saratchandran}{yyy}
    \icmlauthor{Simon Lucey}{yyy}
  \end{icmlauthorlist}

  \icmlaffiliation{yyy}{Australian Institute for Machine Learning, Adelaide University, Australia}

  \icmlcorrespondingauthor{Jianqiao Zheng}{jianqiao.zheng@adelaide.edu.au}

  \icmlkeywords{Implicit Neural Representations, Neural Tangent Kernel, Initialization}

  \vskip 0.3in
]



\printAffiliationsAndNotice{}  
\begin{abstract} Implicit Neural Representations (INRs) have revolutionized continuous signal modeling, yet they struggle to recover fine-grained details within finite training budgets. While empirical techniques, such as positional encoding (PE), sinusoidal activations (SIREN), and batch normalization (BN), effectively mitigate this, their theoretical justifications are predominantly \textit{post hoc}, focusing on the \textit{global} NTK spectrum only \textit{after} modifications are applied. In this work, we reverse this paradigm by introducing a structural diagnostic framework. By performing a layer-wise decomposition of the NTK, we mathematically identify the \textit{``Inlet Rank Collapse''}: a phenomenon where the low-dimensional input coordinates fail to span the high-dimensional embedding space, creating a fundamental rank deficiency at the first layer that acts as an expressive bottleneck for the entire network. This framework provides a unified perspective to re-interpret PE, SIREN, and BN as different forms of rank restoration. Guided by this diagnosis, we derive a \textit{Rank-Expanding Initialization}, a minimalist remedy that ensures the representation rank scales with the layer width without architectural modifications or computational overhead. Our results demonstrate that this principled remedy enables standard MLPs to achieve high-fidelity reconstructions, proving that the key to empowering INRs lies in the structural optimization of the initial rank propagation to effectively populate the latent space. \end{abstract}

\section{Introduction}
\label{sec:intro/first}

\begin{figure}[t]
    \centering
    \subfloat[\centering Default ReLU (Low-rank inlet)]{\includegraphics[width=0.45\linewidth]{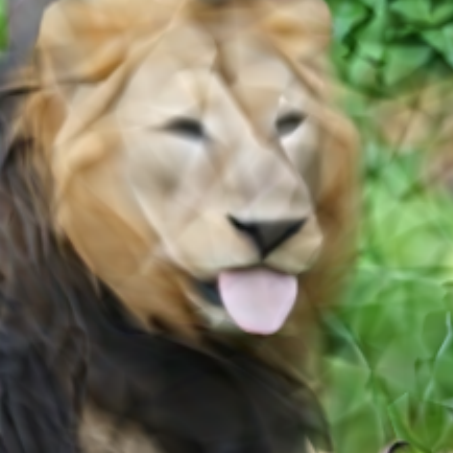}}%
    \qquad
    \subfloat[\centering Rank-Expanding Init (High-rank inlet)]{\includegraphics[width=0.45\linewidth]{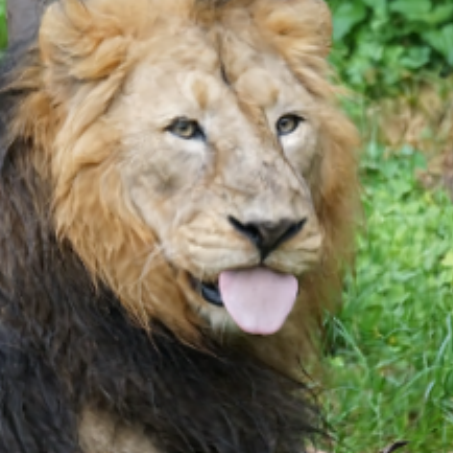}}%
    \caption[Inlet Rank Restoration Unlocks ReLU MLP Capacity.]{
    \textbf{Inlet Rank Restoration Unlocks ReLU MLP Capacity.} 
    Reconstruction results of a standard ReLU-activated MLP. By modifying \textit{only} the initialization of the first layer to expand representation rank, the PSNR improves significantly from $24.14$ (left) to $27.69$ (right). This pure initialization-based gain, without any architectural changes, identifies the Inlet Rank Collapse as the primary bottleneck to the capacity of vanilla INRs.}
    \label{fig:front/first}%
\end{figure}

Implicit Neural Representations (INRs) have emerged as a powerful paradigm for representing continuous signals, such as 3D scenes in neural radiance fields (NeRF)~\cite{mildenhall2021nerf} and high-resolution images, by parameterizing them as coordinate-based MLPs. Despite their success, these models suffer from a fundamental challenge known as ``spectral bias,''~\cite{rahaman2019spectral} prioritizing low-frequency components and struggling to capture fine-grained details within a finite training budget. 

To date, the literature has focused on empirical heuristics to overcome these bottlenecks, including positional encoding (PE)~\cite{tancik2020fourier}, periodic activations (SIREN)~\cite{sitzmann2020implicit}, and architectural modifications like Batch Normalization (BN)~\cite{cai2024batch}. While effective, their theoretical justifications are predominantly \textit{post hoc}. Most studies analyze the Neural Tangent Kernel (NTK)~\cite{jacot2018neural} as a \textit{global} black box only \textit{after} modifications are applied. Consequently, it remains structurally unclear \textit{where} the failure originates.

In this work, we propose a Structural Diagnostic Paradigm for analyzing the optimization dynamics of neural architectures. Unlike conventional studies, our framework enables a layer-wise dissection of the NTK, allowing us to surgically localize where the optimization potential for high-frequency signals is compromised. While generalizable to broader architectures, we apply this tool to INRs to uncover a fundamental flaw: the Inlet Rank Collapse. We demonstrate that since input coordinates reside in a low-dimensional space, the first-layer Jacobian is inherently restricted to a low-rank manifold under standard initializations. This acts as an optimization gatekeeper, preventing the coordinates from effectively populating the high-dimensional latent space and thus attenuating the gradient flow for complex signal components at the source.

Our diagnostic framework provides a unified theoretical lens for the INR landscape. We prove that PE, SIREN, and BN are essentially various forms of rank restoration, designed to bypass the inlet bottleneck. Building on this, we derive a Rank-Expanding Initialization, a minimalist, architecture-agnostic remedy that requires no modifications and zero computational overhead. By resolving the rank propagation issue from the first layer, we empower standard MLPs to achieve high-fidelity representation previously thought to require complex architectural ``band-aids.''

Our main contributions are as follows:
\begin{itemize}[leftmargin=*,noitemsep] \item \textbf{Structural Diagnostic Framework:} We introduce a layer-wise NTK decomposition paradigm that enables a principled analysis of each block's contribution to the global gradient manifold, allowing for the surgical localization of optimization bottlenecks. \item \textbf{Inlet Rank Collapse:} We formalize the \textit{Inlet Rank Collapse}, a phenomenon where low-dimensional coordinate inputs fail to occupy the high-dimensional embedding space, creating a rank deficiency that stifles the propagation of high-bandwidth information. \item \textbf{Unified Theory and Rank-Expanding Remedy:} We provide a unified re-interpretation of PE, SIREN, and BN as various forms of rank restoration. Building on this, we derive a \textbf{Rank-Expanding Initialization}, a minimalist, zero-overhead remedy that facilitates high-rank spatial occupancy at the inlet to unlock the network's capacity. \end{itemize}

\section{Related Work}

\textbf{Implicit Neural Representations.} INRs parameterize signals as continuous functions, achieving remarkable success in shape modeling \cite{park2019deepsdf} and neural rendering \cite{mildenhall2021nerf}. To capture high-fidelity details, existing empirical techniques can be categorized into three types: \textbf{i) Input embedding}, such as Random Fourier Features \cite{tancik2020fourier}, general positional encoding~\cite{zheng2021rethinking}, Hash Encoding \cite{muller2022instant}, and DINER \cite{Xie_2023_CVPR}; \textbf{ii) Alternative activation functions}, such as periodic activations (SIREN) \cite{sitzmann2020implicit}, Gaussian \cite{ramasinghe2022beyond}, H-SIREN \cite{gao2024h}, and FINER \cite{Liu_2024_CVPR}; and \textbf{iii) Architectural modifications}, such as Batch Normalization (BN) \cite{cai2024batch} and high-order methods like HOIN \cite{chen2024hoinhighorderimplicitneural}. However, these methods are often treated as disparate architectural choices. In contrast, our work provides a unified theoretical framework to understand their collective success through the lens of rank restoration.

\textbf{NTK and Spectral Bias.} The Neural Tangent Kernel (NTK) \cite{jacot2018neural} has been instrumental in analyzing the learning dynamics of deep networks. Previous studies \cite{rahaman2019spectral, cao2019towards} identified ``spectral bias'' as a fundamental limitation where networks prioritize low-frequency components. While several aforementioned methods \cite{tancik2020fourier, cai2024batch, chen2024hoinhighorderimplicitneural, Liu_2024_CVPR} leverage NTK to explain their efficacy, their analyses are predominantly \textit{global} and \textit{post hoc}. We extend this by introducing a \textit{layer-wise} decomposition, allowing us to surgically localize the optimization bottleneck at the initial feature Jacobian.

\textbf{Initialization Strategies.} Proper initialization is crucial for the convergence of deep MLPs. While standard strategies like Xavier \cite{glorot2010understanding} and Kaiming \cite{he2015delving} focus on stabilizing signal variance, they do not account for the rank-deficiency issues inherent in coordinate-based inputs. Recent studies have explored specialized initializations for INRs \cite{sitzmann2020implicit}, yet they do not address the structural rank collapse. Our proposed Rank-Expanding Initialization fills this gap by ensuring that the initial feature Jacobian possesses sufficient rank to transmit high-bandwidth information from the very first layer.

\section{Structural Diagnostic Framework}
\label{sec:jac/first}

In this section, we introduce a diagnostic framework designed to localize optimization bottlenecks within neural architectures. Unlike conventional global analysis, we explicitly distinguish two types of derivatives to facilitate a structured decomposition of the Neural Tangent Kernel (NTK): \textbf{Feature Jacobians} with respect to intermediate representations, and \textbf{Weight Jacobians} with respect to model parameters. We demonstrate that the total NTK rank is fundamentally bounded by these layerwise contributions, enabling a fine-grained analysis of how rank collapse propagates through hidden layers.

\subsection{Preliminaries for the Neural Tangent Kernel}
\label{sec:pre_ntk/first}
Consider a neural network $f\:{:}\:\mathbb{R}^{P} \:{\rightarrow}\: \mathbb{R}^{M}$ evaluated on $N$ input samples
$\X \:{\in}\: \mathbb{R}^{N \times P}$ with learnable parameter $\theta$.
The Neural Tangent Kernel is defined as
\begin{equation}
    \label{equ:ntk_kernel/first}
    \K
    =
    \left( \nabla_{\theta} f(\X; \theta) \right)^\top
    \nabla_{\theta} f(\X; \theta),
\end{equation}
where $f(\X; \theta)$ denotes the network output and $\nabla_{\theta} f$ is the derivative of the output with respect to the parameters.

When the network width tends to infinity and training is performed using gradient descent with a sufficiently small learning rate $\eta$ under an $\ell_2$ loss, the training dynamics can be approximated by a linear kernel regression model.
In the discrete-time setting, the network output after $n$ training steps satisfies~\cite{lee2019wide}
\begin{equation}
    \label{equ:ntk_convergence/first}
    \Y^{(n)} - \hat{\Y}
    \approx
    (\I - \eta \K)^n
    (\Y^{(0)} - \hat{\Y}),
\end{equation}
where $\I$ is the identity matrix, $\Y^{(0)}$ is the initial output, and $\hat{\Y}$ is the target output.
Since the NTK matrix admits an eigendecomposition $\K \:{=}\: \Q \Lambda \Q^\top$, the error evolution can be written as
\begin{equation}
    \label{equ:error/first}
    \Y^{(n)} - \hat{\Y}
    \approx
    \Q (\I - \eta \Lambda)^n \Q^\top
    (\Y^{(0)} - \hat{\Y}).
\end{equation}
This expression shows that the initial error is decomposed along the eigenspaces of $\K$, and each eigendirection decays at a rate determined by its corresponding eigenvalue.

\subsection{Distinguishing Feature and Weight Jacobians}
\label{sec:def_Jg}

The internal structure of the NTK can be made explicit by applying the chain rule, which reveals the interplay between state transitions and parameter influences. We formalize this by defining the following two quantities:

\begin{definition}
\textbf{Feature Jacobian.} We define $\J^{i}_{j} \coloneqq \frac{\partial \, \mathrm{vec}(\Z_{i})}{\partial \, \mathrm{vec}(\Z_{j})}$ as the derivative of the $i$-th layer output with respect to the representation at layer $j$. For adjacent layers, we denote $\J_{i} \coloneqq \J^{i}_{i-1}$ as the \textit{layerwise Feature Jacobian}, which dictates how information rank is transmitted from $\Z_{i-1}$ to $\Z_{i}$.
\end{definition}

\begin{definition}
\textbf{Weight Jacobian.} We define $\G^{i}_{j} \coloneqq \frac{\partial \, \mathrm{vec}(\Z_{i})}{\partial \, \mathrm{vec}(\theta_{j})}$ as the derivative of the representation at layer $i$ with respect to the parameters of layer $j$. For the $i$-th layer itself, we denote $\G_{i} \coloneqq \G^{i}_{i}$ as the \textit{layerwise Weight Jacobian}, which represents the direct sensitivity of the layer to its own parameters.
\end{definition}

Under this notation, the overall parameter gradient matrix $\G^L_{\text{all}}$, which forms the backbone of the NTK in \cref{equ:ntk_kernel/first}, can be systematically decomposed.

\subsection{The Rank Constraint Principle}
\label{sec:jac_layer/first}
The overall gradient matrix $\G^{L}_{\mathrm{all}}$ can be written as the horizontal concatenation of gradients with respect to the parameters of individual layers:
\begin{equation}
\label{equ:G_all_def}
    \G^{L}_{\mathrm{all}}
    =
    \left[
        \G^L_L \;\;
        \G^L_{L-1} \;\;
        \cdots \;\;
        \G^L_{1}
    \right].
\end{equation}

Each matrix $\G^L_k$ has the same number of rows as $\mathrm{vec}(\Z_{L})$, while its number of columns equals the number of parameters in $\theta_k$.

\begin{proposition}
\label{prop:concatenation/first}
The (exact) rank of the overall gradient matrix $\G^{L}_{\mathrm{all}}$ satisfies
\begin{equation}
\label{equ:concatenation/first}
    \max_{k=1,\dots,L} \mathrm{rank}(\G^{L}_{k})
    \;\leq\;
    \mathrm{rank}(\G^{L}_{\mathrm{all}})
    \;\leq\;
    \sum_{k=1}^{L} \mathrm{rank}(\G^{L}_{k}).
\end{equation}
\end{proposition}

\paragraph{Discussion.}
The lower bound follows from the fact that the concatenated matrix must preserve at least the rank of its most expressive block.
The upper bound corresponds to the idealized case in which the column spaces of the matrices $\G^L_k$ are mutually independent.

\subsection{Layerwise Decomposition via the Chain Rule}
\label{sec:chain}

Each component $\G^L_k$ of the total gradient matrix admits a further decomposition, revealing how the parameter sensitivity of an internal layer is propagated to the final output:
\begin{equation}
\label{equ:chain}
    \G^L_k = \left( \prod_{i=L}^{k+1} \J_{i} \right) \G_{k}.
\end{equation}

\begin{proposition}
\label{prop:decompose}
The overall parameter gradient $\G^{L}_{\mathrm{all}}$ admits a structured layerwise decomposition:
\begin{equation}
\label{equ:G_all_decomp}
    \G^{L}_{\mathrm{all}} = \left[ \G_L,\; \J_L \G_{L-1},\; 
    \dots,\; \left( \prod_{i=L}^{2} \J_i \right) \G_1 \right].
\end{equation}
\end{proposition}

\begin{notebox}
    \paragraph{Structural Insight.} \cref{equ:G_all_decomp} provides a transparent view of the NTK's anatomy. It reveals that the optimization potential of any layer $k$ is not merely a function of its local parameters, but is \textit{filtered} by the chain of Feature Jacobians $\J_i$ from all subsequent layers. This formulation suggests that if any component in the chain possesses a restricted rank, it will act as a structural bottleneck, effectively strangling the gradient flow for all preceding parameters.
\end{notebox}

\section{The Inlet Rank Collapse in INRs}
\label{sec:bottleneck_mlp}

In this section, we apply the diagnostic framework developed in \cref{sec:jac/first} to coordinate-based MLPs. We demonstrate that the fundamental optimization bottleneck is not caused by the parameters themselves, but by the low-rank intermediate representations that starve the weight gradients at the network's inlet.

\subsection{Preliminaries for MLP Networks}
\label{sec:pre_mlp/first}

We consider a coordinate-based MLP composed of an input embedding $\phi$, a stack of $L$ hidden layers, and a linear output head. Let $\X \in \mathbb{R}^{N \times P}$ be the input coordinates. The network transforms $\X$ into a sequence of intermediate representations $\Z_i \in \mathbb{R}^{N \times D}$:
\begin{equation}
    \begin{aligned}
        \Z_{0} &= \phi(\X), \\
        \Z_{i} &= \rho(\Z_{i-1} \W_{i} + \mathbf{b}_{i}), \quad i = 1, \dots, L, \\
        \Y     &= \Z_{L} \W_{\mathrm{out}} + \mathbf{b}_{\mathrm{out}},
    \end{aligned}
\end{equation}
where $\W_i \in \mathbb{R}^{D \times D}$ and $\W_{\mathrm{out}} \:{\in}\: \mathbb{R}^{D \times M}$ are learnable weights and $\rho$ is a non-linear activation (e.g., ReLU or $\sin$). The embedding $\phi: \mathbb{R}^{P} \to \mathbb{R}^{D}$ maps low-dimensional coordinates to the hidden space via linear projection, sinusoidal activation (e.g., SIREN)~\cite{sitzmann2020implicit}, batch normalization~\cite{cai2024batch}, or positional encoding such as Random Fourier Features~\cite{tancik2020fourier}.

\subsection{Jacobians and Gradients in MLPs}
\label{sec:jac_emb/first}

To pinpoint the structural bottleneck, we first derive the closed-form expressions for Jacobians and gradients in a linear MLP. 

\begin{lemma}
\label{lem:vec_linear}
For a linear layer $\Z_{i} = \Z_{i-1} \W_{i}$, the Feature Jacobian $\J_{i}$ and the Weight Jacobian $\G_{i}$ exhibit the following Kronecker product structures:
\begin{equation}
\label{equ:d_layer}
    \J_{i} = \W_{i}^{\top} \otimes \I_{N}, \quad \G_{i} = \I_{D} \otimes \Z_{i-1}.
\end{equation}
\end{lemma}

By combining \cref{equ:d_layer} with the decomposition in \cref{equ:chain}, we obtain the final form of the gradient for each layer's parameters, which reveals a critical coupling:

\begin{corollary}
\label{cor:G_L_k_form}
For a linear MLP, the gradient of the output with respect to the $k$-th layer parameters is:
\begin{equation}
\label{equ:G_final_form}
    \G^{L}_{k} = \underbrace{\left( \I_{D} \prod_{i=L}^{k+1} \W_{i} \right)^{\top}}_{\text{Weight Product Factor}} \otimes \underbrace{\Z_{k-1}}_{\text{Representation Factor}}.
\end{equation}
\end{corollary}

\paragraph{Extension to Non-linearities.} While activations (e.g., ReLU) or Batch Normalization introduce diagonal modulation matrices, they do not alter the fundamental Kronecker coupling between downstream weight products and upstream representations. We provide the generalized formulations in \textbf{Appendix~\cref{appendix:additional_jac}}. For analytical clarity, we focus on \cref{equ:G_final_form} to localize the primary bottleneck.

\subsection{Diagnosis: Healthy Weights vs. Starved Gradients}

\begin{figure}[t]
\centering
\begin{subfigure}{0.48\linewidth}
    \centering
    \includegraphics[width=\linewidth]{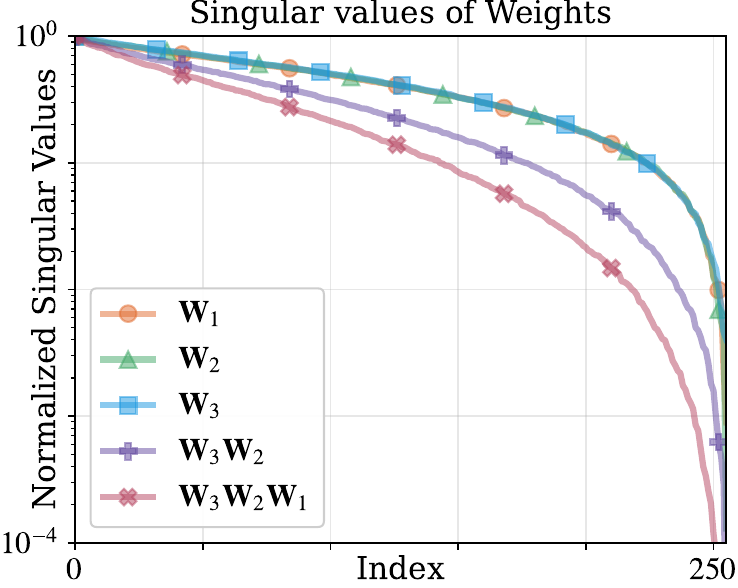}
    \caption{Weight Spectra}
    \label{fig:svd_weights}
\end{subfigure}
\hfill
\begin{subfigure}{0.48\linewidth}
    \centering
    \includegraphics[width=\linewidth]{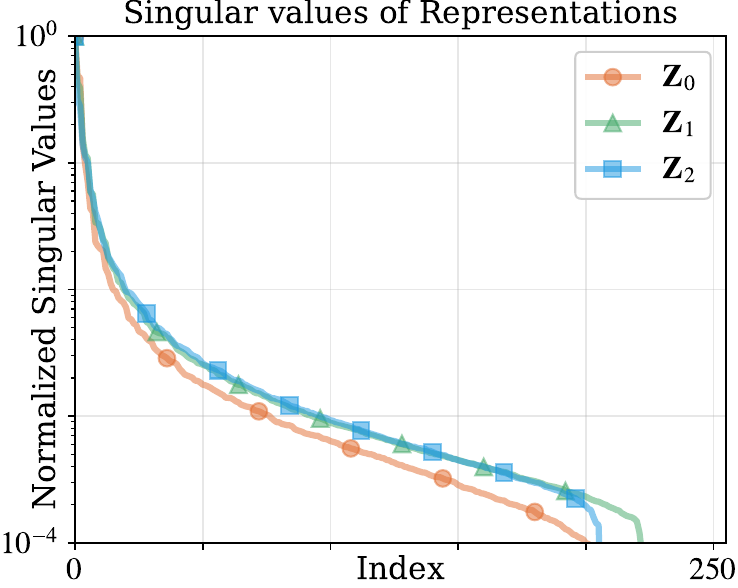}
    \caption{Representation Spectra}
    \label{fig:svd_representations}
\end{subfigure}
\caption{The contrast between (a) weight matrices $\W_i$ and (b) intermediate representations $\Z_i$ in a 3-layer vanilla MLP (width 256). While weights remain full-rank and well-conditioned, the representations suffer from severe \textbf{Inlet Rank Collapse}, effectively starving the parameter gradients.}
\label{fig:svd_comparison}
\end{figure}

We now investigate which factor in \cref{equ:G_final_form} constitutes the dominant bottleneck. Using standard random initialization, we observe a stark contrast between the two factors:

\begin{remark}[Healthy Weights: The Capacity]
\label{remark:random_weights}
As illustrated in \cref{fig:svd_comparison} (a), the \textbf{Weight Product Factor} ($\I_{D} \prod \W_{i}$) is typically full-rank ($D$) and well-conditioned. This suggests that the "optimization machinery" possesses ample degrees of freedom to capture complex signals.
\end{remark}

\begin{remark}[Inlet Rank Collapse: The Starvation]
\label{remark:representations}
In contrast, the \textbf{Representation Factor} ($\Z_{k-1}$) exhibits severe rank deficiency (\cref{fig:svd_comparison} (b)). Because the input $\X$ resides in a low-dimensional coordinate space ($P \in \{2, 3\}$), the initial representation $\Z_0 = \phi(\X)$ is inherently low-rank. This \textbf{``Inlet Rank Collapse''} persists through early layers, failing to expand the signal's dimensionality into the hidden space.
\end{remark}

Combining the layerwise decomposition from \cref{prop:decompose} with the Kronecker property $\operatorname{rank}(A \otimes B) = \operatorname{rank}(A)\operatorname{rank}(B)$, we arrive at our core diagnostic conclusion:

\begin{proposition}
\label{prop:bottleneck_mlp}
In coordinate-based MLPs, the rank of the overall parameter gradient $\G^L_{\mathrm{all}}$, and thus the NTK, is primarily bounded by the ranks of intermediate representations:
\begin{equation}
    \max_k \operatorname{rank}(\Z_{k-1}) \leq \operatorname{rank}(\G^L_{\mathrm{all}}) \leq D \sum_{k=1}^L \operatorname{rank}(\Z_{k-1}).
\end{equation}
\end{proposition}

\begin{notebox}
    \paragraph{Structural Diagnosis.} Our analysis reveals that weights in coordinate MLPs are \textbf{``starved''}. Despite the high-dimensional parameter space, the available gradient directions are trapped within the low-rank manifold defined by $\Z_{k-1}$. This explains why increasing depth or width often yields marginal gains: the "inlet" is too narrow to accommodate high-bandwidth information, causing the entire NTK to inherit a rank-deficient structure that stifles.
\end{notebox}



\begin{figure*}[ht]
    \centering
    \includegraphics[width=1.0\linewidth]{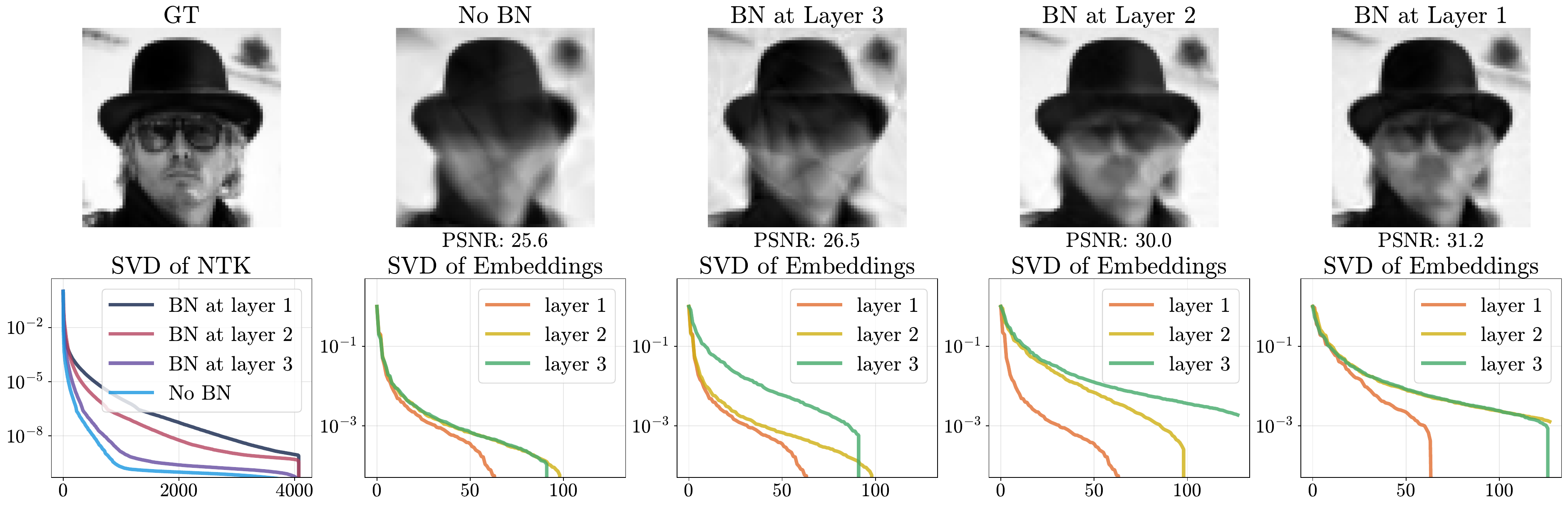}
    \caption{Reconstruction quality vs. BN position. The experiment confirms that addressing rank deficiency at the \textbf{inlet} (Layer 1) is the only way to fully recover the network's effective capacity. Interventions in deeper layers cannot recover the gradient diversity of preceding stages, as those earlier weights remain trapped in the low-rank subspace imposed by the inlet.}
    \label{fig:layer_bn/first}
\end{figure*}
\begin{figure}[t]
\centering
\includegraphics[width=0.9\linewidth]{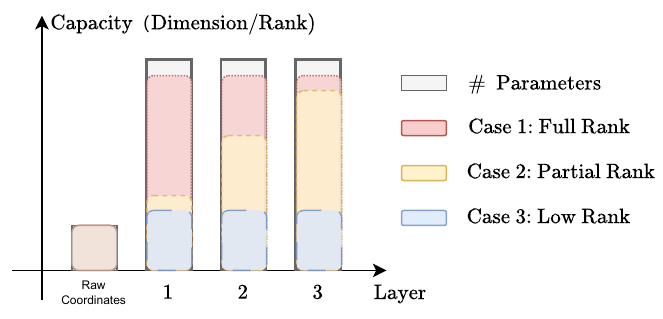}
\caption{Conceptual illustration of capacity. While parameter dimension (gray) is fixed, effective capacity is governed by representation rank (colored). \textbf{Case 3}: Persistent low rank at the inlet starves the entire network. \textbf{Case 2}: Mid-stage expansion only recovers capacity for deeper layers. \textbf{Case 1}: Early rank expansion at the inlet unlocks the capacity of all parameters.}
\label{fig:capacity/first}
\end{figure}

\section{A Unified Perspective on Existing Remedies}
\label{sec:unified_perspective}

Building upon the diagnostic of inlet rank collapse, we now provide a capacity-level interpretation. While classical model capacity is often equated with parameter count, our analysis reveals that the \textbf{effective capacity} of coordinate MLPs is fundamentally capped by the numerical rank of intermediate representations. This perspective allows us to localize the optimization bottleneck to the very first layer and provides a unified explanation for how diverse empirical techniques (PE, SIREN, BN) all fundamentally serve to mitigate this same structural deficiency.

\subsection{Effective Capacity vs. Parameter Count}
\label{sec:thought_experiment}

To illustrate the decoupling of parameter count and learning power, we consider some scenarios with identical architectures but differing representation rank dynamics (\cref{fig:capacity/first}):

\paragraph{The Starvation Phenomenon (Case 3):} In vanilla coordinate MLPs, the rank is locked to the low-dimensional input at the inlet. Despite having high-dimensional weight matrices (gray bars), the gradient manifold is so restricted that the vast majority of parameters remain \textbf{functionally dormant}. This reveals that without sufficient representation rank, raw parameter count is a misleading metric for model capacity.
\paragraph{Unlocking Potential (Case 1):} Conversely, when the representation achieves near-full rank at the inlet, the high-dimensional weight matrices at every layer are "fed" with a sufficiently rich basis. This ensures that the model's effective capacity (colored bars) matches its total parameterization, unlocking the full expressive power of the architecture.

\subsection{Experimental Validation: The Primacy of the Inlet}
\label{sec:exp_evidence}

To demonstrate that the network's effective capacity is gated at the inlet, we compare four settings of a 3-layer MLP where a Batch Normalization (BN) layer is inserted at different stages to serve as a local rank-restoration module (\cref{fig:layer_bn/first}). This setup allows us to observe how the stage of rank restoration dictates the utilization of the entire network's parametric capacity. The results confirm that the "starvation" of weights must be addressed at the earliest possible stage to unlock the model's full potential.

\paragraph{Inlet Restoration (BN at Layer 1):} Restoring the rank immediately after the initial dimension expansion feeds all subsequent weight Jacobians with high-rank signals. This single-point intervention at the inlet is sufficient to unlock the latent capacity of the entire architecture, as it ensures that the high-dimensional weight matrices throughout the network have access to a diverse basis for optimization.
    
\paragraph{Deep Intervention (BN at Layer 3):} While the rank is increased locally for the final layer, the preceding weights remain starved. The improvement is marginal because a late-stage intervention cannot rescue the gradient directions already lost at the inlet, leaving the majority of the network's parameters restricted and under-utilized.

\subsection{Unified View: Rank Restoration at the Inlet}

This capacity-centric perspective provides a unified explanation for the success of various INR techniques, all of which function as structural remedies for the inlet bottleneck.

\paragraph{Positional Encoding (PE):} An \textit{explicit inlet expansion} that inflates $\operatorname{rank}(\Z_0)$ via Fourier mapping before the signal enters the first trainable layer. By manually lifting the input into a high-dimensional space, PE bypasses the inlet rank collapse entirely, ensuring that all subsequent weights are "fed" with high-rank representations from the outset.

\paragraph{SIREN and Batch Normalization:} While these methods are often applied globally, their primary effectiveness stems from their ability to increase the representation rank immediately after the initial dimension expansion. Whether through the high-frequency nature of sinusoidal activations or the spectral spreading of normalization, these techniques act as \textit{implicit rank-restorers} at the inlet. Once the rank is established at the start, the remaining layers, which behave similarly to standard ReLU layers, simply preserve this high-rank flow.

\begin{notebox} 
\paragraph{Core Insight: The Inlet Dominance Effect}
The effective capacity of an INR is determined by its representation rank rather than raw parameter count. The earlier the rank is expanded, the more parameters are unlocked throughout the network. Most existing techniques succeed precisely because they intervene at the first layer to resolve the \textbf{Inlet Rank Collapse}.
\end{notebox}

\section{Rank-Expanding Initialization}
\label{sec:init/first}

Based on our preceding analysis, the input embedding layer is the primary bottleneck in ReLU-based coordinate MLPs. Under standard random initialization, a ReLU layer fails to expand the numerical rank of representations as the embedding dimension grows. Consequently, the rank deficiency introduced at the inlet persists throughout the network, locking the model into a low-capacity regime.

While most existing remedies introduce new architectural components, such as Fourier features or periodic activations, our theory suggests a more fundamental solution. If representation rank alone governs effective capacity, it should be possible to "unlock" the network without altering its architecture:

\begin{center}
\emph{Can we restore the rank of the input embedding purely through initialization?}
\end{center}

We show that the answer is affirmative. By geometrically designing the initial weights $\W^{*}$ and biases $\mathbf{b}^{*}$, we can force the input ReLU layer to generate a high-rank representation. This approach isolates rank enhancement from other confounding architectural factors, providing a "clean" validation of the Inlet Dominance Effect.

\subsection{1-D Case: Constructing Triangular Representations}
\label{sec:1d_init/first}

In the 1-D setting, the number of input coordinates $N$ is typically manageable, allowing the embedding dimension $D$ to scale accordingly. To achieve full rank, we leverage a construction adapted from \cite{zhang2017understanding} that ensures a full-rank ReLU embedding.

For sorted 1-D inputs $\x = [x_{1}, \dots, x_{N}]$, we set $\W^{*} = \mathbf{1}$ and $\mathbf{b}^{*} = [-x_1 + \epsilon, \dots, -x_N + \epsilon]^\top$, where $\epsilon > 0$ is a small constant. As shown in \cref{fig:init_1d/first}, this configuration ensures that each neuron $j$ is active only for $x \geq x_j - \epsilon$. This yields a \textbf{lower-triangular activation pattern} in $\Z_0$, which is guaranteed to be full-rank, effectively bypassing the starvation problem from the very first layer.

\begin{figure}[ht]
    \centering
    \includegraphics[width=0.8\linewidth]{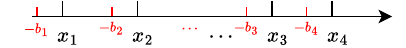}
    \caption[Illustration of initialization in the 1-D case.]
    {Illustration of the proposed initialization in the 1-D case.
    Carefully chosen biases separate activation thresholds across input
    coordinates, yielding a full-rank, triangular representation matrix.}
    \label{fig:init_1d/first}
\end{figure}

\begin{figure}[ht]
    \centering
    \includegraphics[width=0.6\linewidth]{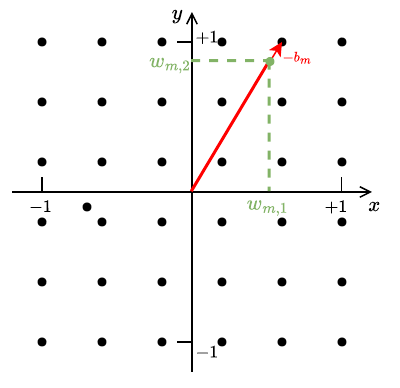}
    \caption[Illustration of initialization in the 2-D case.]
    {Illustration of the proposed initialization in the 2-D case.
    Each grid point $\V$ defines a normalized weight vector and a bias equal to the
    negative norm plus $\epsilon$.
    This geometry spreads ReLU thresholds across space, facilitating
    a high-rank input embedding.}
    \label{fig:init_2d/first}
\end{figure}

\begin{figure*}[t]
    \centering
    \subfloat[\centering SVD of embeddings]
    {\includegraphics[width=0.45\linewidth]{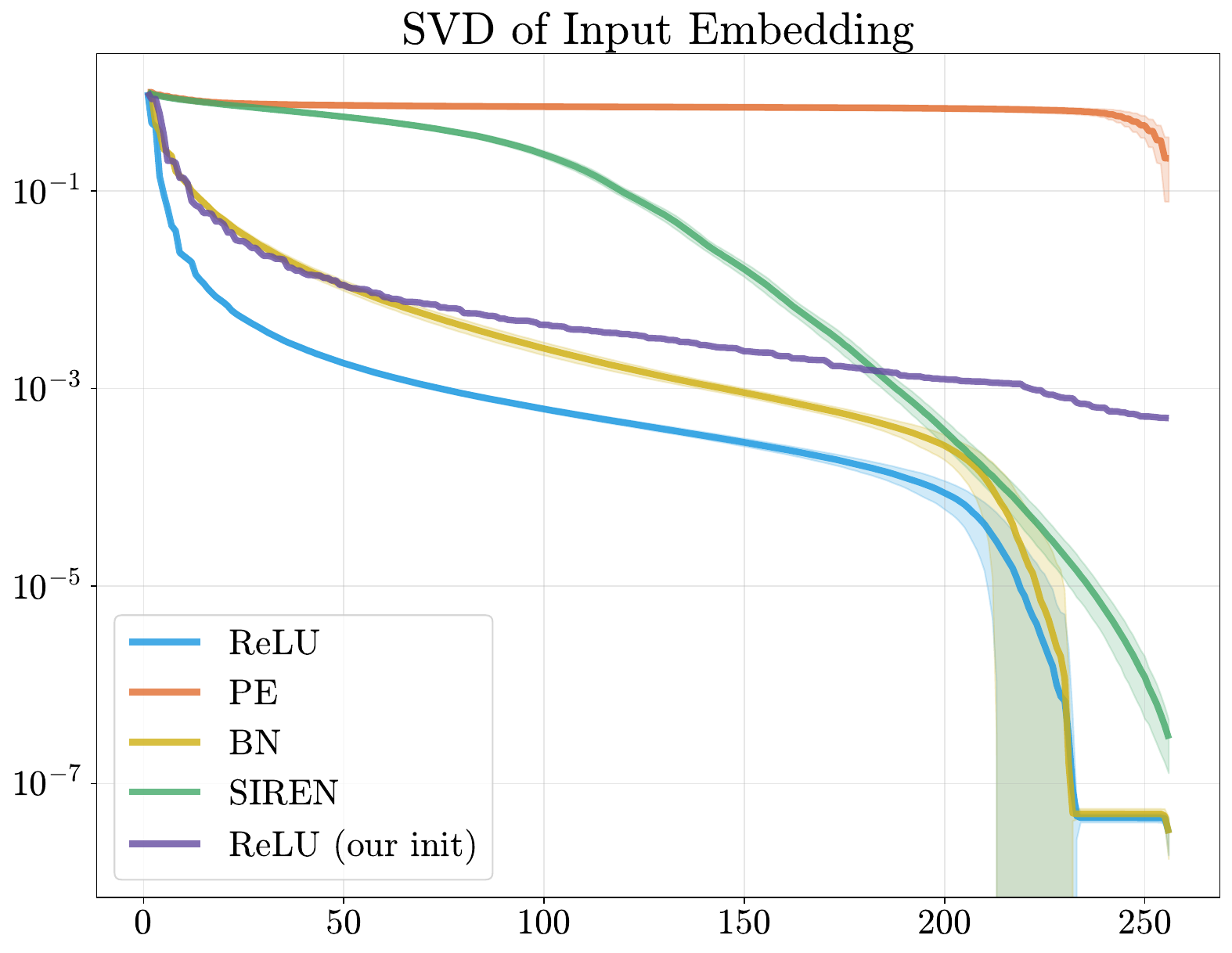}}%
    \quad
     \subfloat[\centering SVD of NTK]
     {\includegraphics[width=0.45\linewidth]{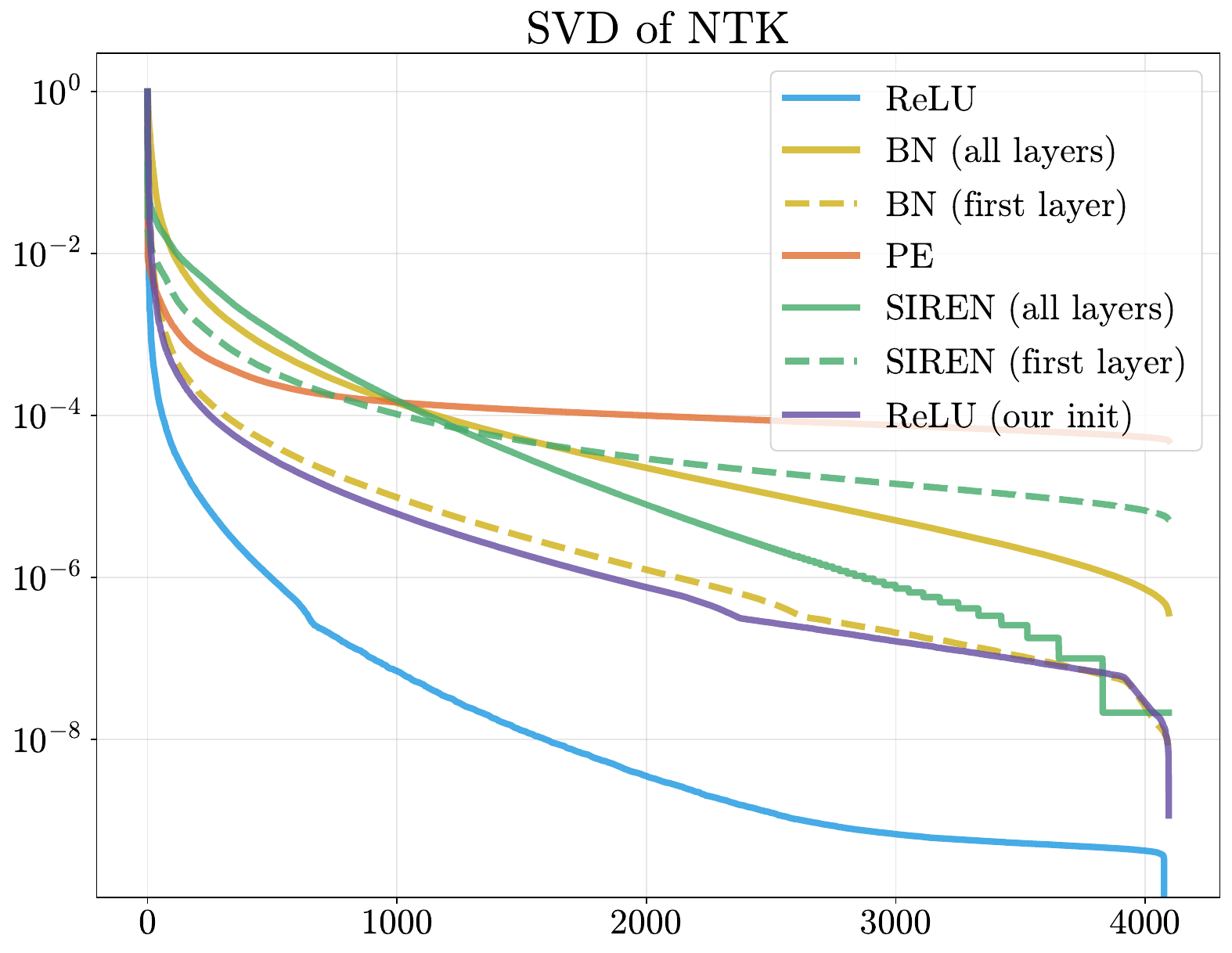}}%
    \caption[Comparison of SVD of input embeddings and NTK kernel with various methods.]{(A) Singular value spectra of input embeddings under different methods.
    Our initialization substantially increases the numerical rank of the input embedding.
    Shaded regions indicate the standard deviation over $100$ runs.
    (B) Singular value spectra of the corresponding NTK kernels.
    Improving the rank of the first-layer representations leads to a significantly
    better-conditioned NTK spectrum.
    Dashed curves correspond to global methods applied only at the first layer,
    with all subsequent layers remaining standard ReLU layers.}%
    \label{fig:svd_embedding_ntk/first}%
\end{figure*}

\subsection{2-D Case: Geometric Threshold Spreading}
\label{sec:2d_init/first}

In 2-D, the number of coordinates grows quadratically, making it impractical to assign a unique neuron to every pixel. Instead, we propose a geometric strategy to distribute ReLU activation boundaries across the input domain.

\begin{proposition}
\label{prop:2d_init/first}
Let a ReLU linear layer take normalized two-dimensional coordinates
$\X \:{\in}\: [-1,1]^2$, and let the embedding dimension be $D \:{=}\: p^2$.
Define a set of grid points
\begin{equation}
    \V_{p(i-1)+j}
    =
    \left[
    \frac{2i}{p+1} - 1,\;
    \frac{2j}{p+1} - 1
    \right],
\end{equation}
for $i,j \:{=}\: 1,\ldots,p$.
For each grid point, define the weight and bias as
\begin{equation}
\begin{aligned}
    \W^{*}_{p(i-1)+j} &= \frac{\V_{p(i-1)+j}}{\|\V_{p(i-1)+j}\|}, \\
    \mathbf{b}^{*}_{p(i-1)+j} &= -\|\V_{p(i-1)+j}\| + \epsilon,
\end{aligned}
\end{equation}
where $\epsilon > 0$ is a small constant.
\end{proposition}

By spreading these boundaries across the domain (\cref{fig:init_2d/first}), we ensure diverse activation patterns. Unlike random initialization, which often leaves large regions of the coordinate space "dead" or redundant, this geometric spreading facilitates a high-rank embedding that captures spatial variations more effectively.

\begin{figure}[t]
    \centering
    \subfloat[\centering ReLU (default):\\90.98]
    {\includegraphics[width=0.24\linewidth]{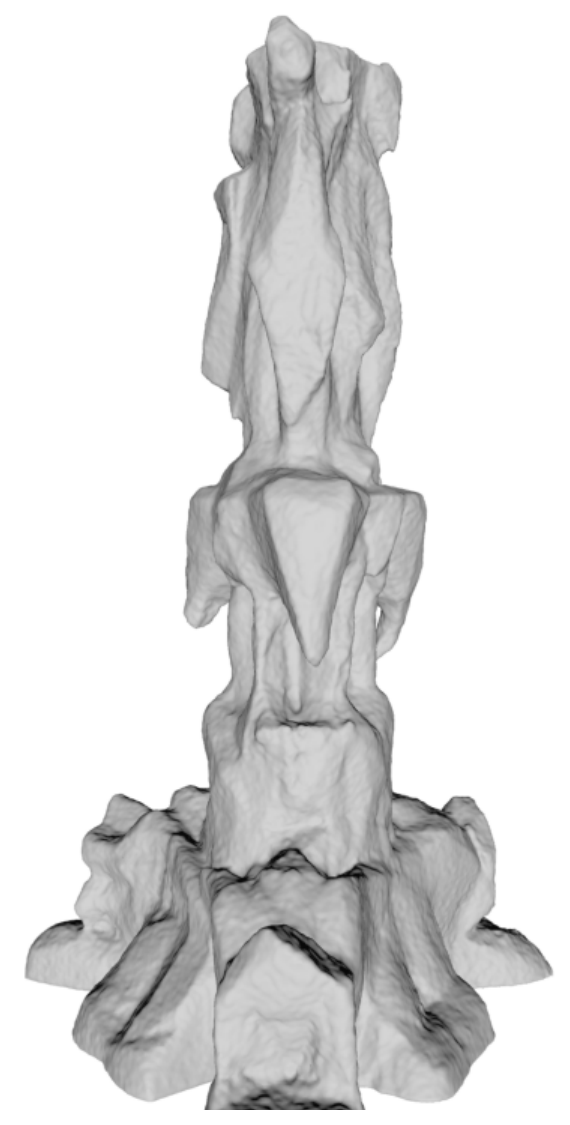}}%
     \subfloat[\centering PE:\\97.50]
     {\includegraphics[width=0.24\linewidth]{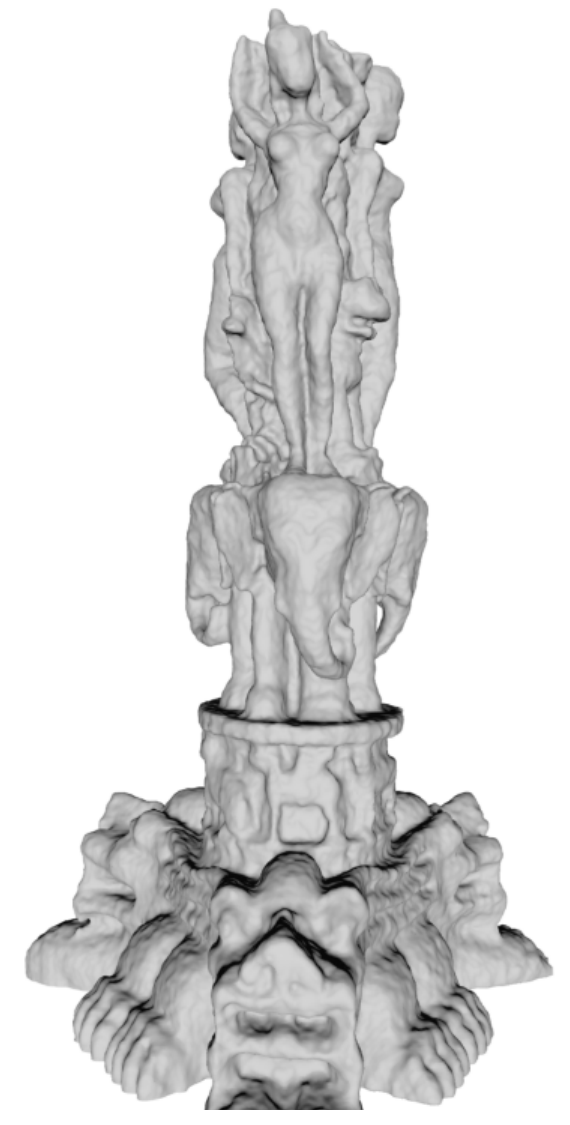}}%
     \subfloat[\centering SIREN:\\96.36]
     {\includegraphics[width=0.24\linewidth]{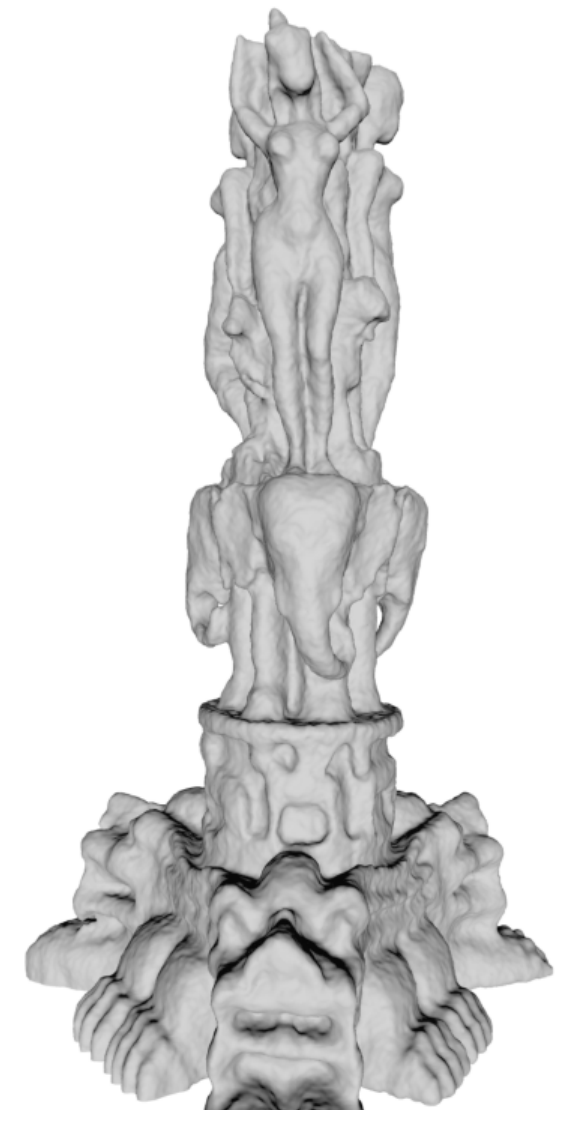}}%
     \subfloat[\centering SIREN \\(first layer):\\97.01]{\includegraphics[width=0.24\linewidth]{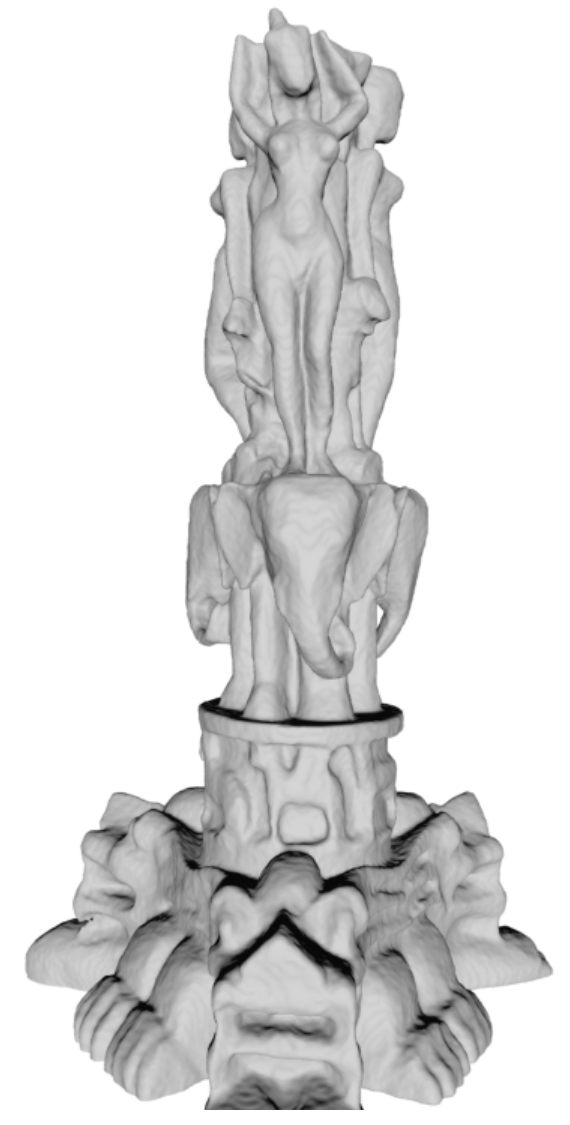}}%
     \quad
     \subfloat[\centering BN:\\95.17]
     {\includegraphics[width=0.24\linewidth]{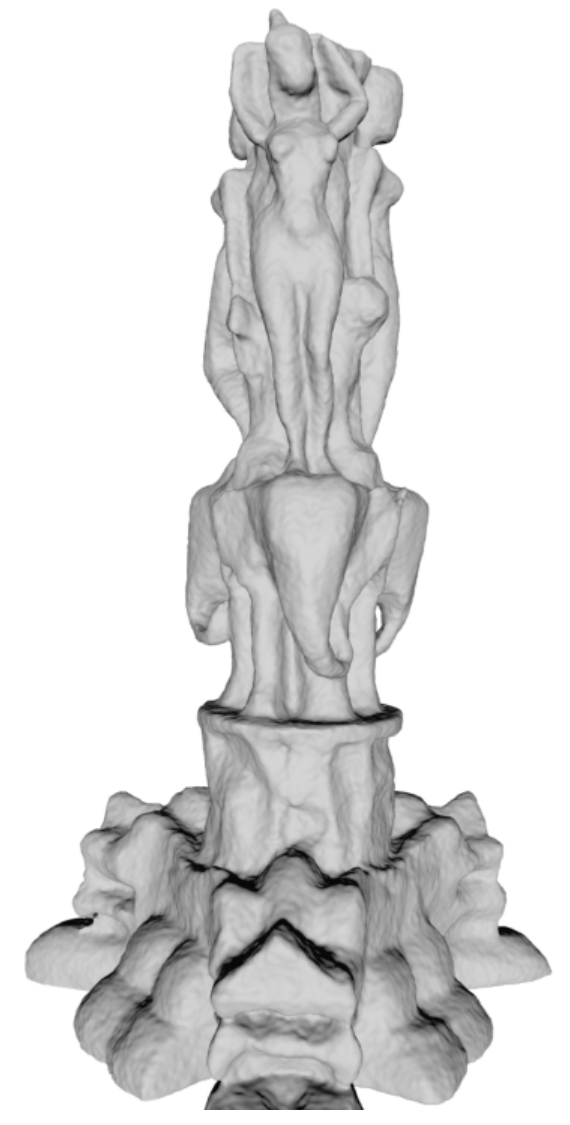}}%
     \subfloat[\centering BN \\(first layer):\\96.15]
     {\includegraphics[width=0.24\linewidth]{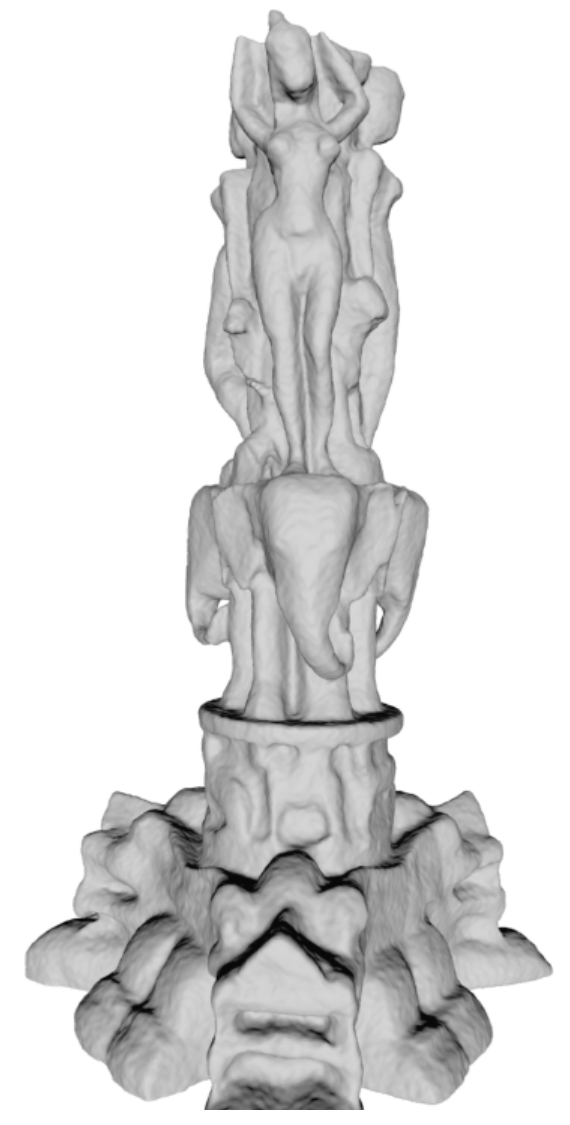}}%
     \subfloat[\centering ReLU \\(our init):\\95.75]
     {\includegraphics[width=0.24\linewidth]{imgs/3d_posenc.pdf}}%
    \caption{Meshes reconstructed from occupancy volumes using different methods,
    along with their intersection-over-union (IoU) scores.}%
    \label{fig:3d_occupancy/first}%
\end{figure}

\section{Experiments}
\label{sec:exp/first}

We evaluate the proposed rank-expanding initialization across three dimensions: spectral analysis of the training manifold, 2-D image reconstruction, and 3-D occupancy modeling. Our goal is to verify that restoring inlet rank is sufficient to unlock the latent capacity of standard ReLU MLPs.

\subsection{Spectral Analysis: Embedding and NTK Rank}

To validate our core claim, we analyze the numerical rank of input embeddings and their corresponding Neural Tangent Kernels (NTK) on a $256 \times 256$ grid (width $256$). We compare our initialization against default ReLU, Positional Encoding (PE), Batch Normalization (BN), and SIREN.

As shown in \cref{fig:svd_embedding_ntk/first}A, our initialization (purple) substantially lifts the singular value spectrum of the inlet embedding, matching the behavior of architectural remedies like PE and SIREN. More importantly, \cref{fig:svd_embedding_ntk/first}B reveals that this first-layer intervention significantly improves the conditioning of the global NTK. This spectral alignment confirms that the "Inlet Rank Collapse" is a primary driver of poor NTK conditioning in vanilla INRs, and that our pure-initialization remedy effectively resolves it.

\subsection{2-D Image Representation}

\begin{table}[t]
\caption[$2$-D image representation PSRN of a collection of 16 images from Div2k.]{$2$-D image representation PSNR averaged over 16 images from the DIV2K
dataset~\cite{Agustsson_2017_CVPR_Workshops}.
``BN'' denotes batch normalization.
``First layer'' indicates that the global method (SIREN or BN) is applied only
to the first layer, while all subsequent layers remain standard ReLU layers.}
\centering
\begin{adjustbox}{width=1.0\linewidth}
\begin{tabular}
{@{}lcccccc@{}}
\toprule
& \thead{\normalsize Method} & \thead{\normalsize ReLU\\(default init)} & \thead{\normalsize PE} & \thead{\normalsize SIREN} & \thead{\normalsize  BN} & \thead{\normalsize ReLU \\(our init)} \\
\midrule
& \thead{\normalsize default\\setting} &23.57 & 27.15 & 27.29  & 25.04 & 25.61 \\
& \thead{\normalsize first\\layer} & -- & -- & 26.15 & 25.82 & -- \\
\bottomrule
\end{tabular}
\label{tab:2d_image/first}
\end{adjustbox}
\end{table}

We evaluate reconstruction performance on 16 images from the DIV2K dataset \cite{Agustsson_2017_CVPR_Workshops}. As reported in \cref{tab:2d_image/first}, applying global methods (SIREN, BN) only to the first layer yields performance nearly identical to their full-network counterparts. 

Notably, our initialization, which introduces zero architectural changes, improves the baseline ReLU PSNR by approximately $2$~dB. This result underscores that the "dormant" parameters in a ReLU MLP are not inherently incapable of representing high frequencies; they simply require a high-rank signal from the inlet to be effectively utilized.

\subsection{3-D Occupancy Representation}

Finally, we test the scalability of our approach on 3-D surface reconstruction. As illustrated in \cref{fig:3d_occupancy/first}, a default ReLU MLP fails to capture complex geometry, resulting in oversmoothed surfaces and low IoU scores. 

By contrast, our initialization enables the same ReLU architecture to recover high-fidelity meshes, achieving competitive IoU scores comparable to SIREN and PE. Consistent with our 2-D findings, the performance of SIREN and BN remains robust even when restricted to the first layer. These results provide strong empirical evidence for the \textbf{Inlet Dominance Effect}: the success of an INR is largely determined by how it handles the coordinate embedding at the very first layer.

\section{Conclusion}

In this work, we introduced a structural diagnostic framework that localizes optimization bottlenecks through layer-wise NTK dissection. Our analysis identifies the Inlet Rank Collapse as the primary obstacle in coordinate-based MLPs: the low-dimensional input create a persistent rank deficiency at the first layer, which acts as a gatekeeper that "starves" subsequent parameters of gradient information.

This framework offers a unified perspective where diverse heuristics, such as Positional Encoding, SIREN, and Batch Normalization, are identified as sharing a common underlying mechanism: the restoration of representation rank. Guided by this insight, we proposed Rank-Expanding Initialization, a minimalist and zero-overhead remedy that unlocks the latent capacity of standard ReLU MLPs without architectural modifications. Our empirical results confirm that resolving rank collapse at its source is sufficient to achieve high-fidelity representation. Beyond INRs, the structural diagnostic paradigm introduced here offers a promising avenue for optimizing information flow in broader neural architectures such as Transformers and CNNs by ensuring rank consistency throughout the network.

\section*{Impact Statement}


This paper presents work whose goal is to advance the field of Machine
Learning. There are many potential societal consequences of our work, none
which we feel must be specifically highlighted here.


\nocite{langley00}

\bibliography{example_paper}
\bibliographystyle{icml2026}

\newpage
\appendix
\onecolumn



\section{Proofs for Jacobians and Gradients}
\label{appendix:proofs}

\subsection{Proof of Lemma~\ref{lem:vec_linear} and Corollary~\ref{cor:G_L_k_form}}
\label{sec:proof_jac}

\begin{proof}
For a linear layer $\mathbf{Z}_{i} = \mathbf{Z}_{i-1} \mathbf{W}_{i}$, we apply the standard vectorization identity $\mathrm{vec}(\mathbf{AXB}) = (\mathbf{B}^{\top} \otimes \mathbf{A})\mathrm{vec}(\mathbf{X})$. 
Setting $\mathbf{A} = \mathbf{I}_N$, $\mathbf{X} = \mathbf{Z}_{i-1}$, and $\mathbf{B} = \mathbf{W}_i$, we obtain:
\begin{equation}
    \mathrm{vec}(\mathbf{Z}_i) = (\mathbf{W}_i^{\top} \otimes \mathbf{I}_N) \mathrm{vec}(\mathbf{Z}_{i-1}).
\end{equation}
Taking the derivative with respect to $\mathrm{vec}(\mathbf{Z}_{i-1})$ yields the layer-wise Jacobian $\mathbf{J}_i = \mathbf{W}_i^{\top} \otimes \mathbf{I}_N$. Similarly, setting $\mathbf{B} = \mathbf{I}_D$, $\mathbf{X} = \mathbf{W}_i$, and $\mathbf{A} = \mathbf{Z}_{i-1}$, we have:
\begin{equation}
    \mathrm{vec}(\mathbf{Z}_i) = (\mathbf{I}_D \otimes \mathbf{Z}_{i-1}) \mathrm{vec}(\mathbf{W}_i),
\end{equation}
which yields the parameter gradient $\mathbf{G}_i = \mathbf{I}_D \otimes \mathbf{Z}_{i-1}$. 

For the $L$-layer linear MLP, the gradient of the final output with respect to the $k$-th layer weights is derived via the chain rule:
\begin{equation}
    \mathbf{G}^{L}_{k} = \frac{\partial \mathrm{vec}(\mathbf{Z}_L)}{\partial \mathrm{vec}(\mathbf{W}_k)} = \left( \prod_{i=L}^{k+1} \frac{\partial \mathrm{vec}(\mathbf{Z}_i)}{\partial \mathrm{vec}(\mathbf{Z}_{i-1})} \right) \frac{\partial \mathrm{vec}(\mathbf{Z}_k)}{\partial \mathrm{vec}(\mathbf{W}_k)}.
\end{equation}
Substituting the expressions for $\mathbf{J}_i$ and $\mathbf{G}_k$:
\begin{equation}
    \mathbf{G}^{L}_{k} = \left( \prod_{i=L}^{k+1} (\mathbf{W}_i^{\top} \otimes \mathbf{I}_N) \right) (\mathbf{I}_D \otimes \mathbf{Z}_{k-1}) = \left( \prod_{i=L}^{k+1} \mathbf{W}_i \right)^{\top} \otimes \mathbf{Z}_{k-1}.
\end{equation}
The last step uses the mixed-product property $(\mathbf{A} \otimes \mathbf{B})(\mathbf{C} \otimes \mathbf{D}) = (\mathbf{AC} \otimes \mathbf{BD})$ and the fact that $(\prod \mathbf{M}_i)^\top = \prod \mathbf{M}_i^\top$ in the context of these Jacobians. This completes the proof.
\end{proof}

\subsection{Jacobians for Nonlinear Activations and Batch Normalization}

In this section, we extend the linear analysis to networks incorporating nonlinear activation functions and Batch Normalization.

\begin{corollary}
\label{cor:J_relu/first}
For a network with element-wise nonlinear activations $\rho$ (e.g., ReLU), the Jacobians and gradients satisfy:
\begin{equation}
\label{equ:J_relu/first}
\begin{aligned}
    \mathbf{J}_i &= \mathbf{D}^{A}_{i}(\mathbf{W}_{i}^{\top} \otimes \mathbf{I}_{N}), \\
    \mathbf{G}_i &= \mathbf{D}^{A}_{i}(\mathbf{I}_{D} \otimes \mathbf{Z}_{i-1}), \\
    \mathbf{G}^{L}_{k} &= \left( \prod_{i=L}^{k+1} \mathbf{D}^{A}_{i}(\mathbf{W}_{i}^{\top} \otimes \mathbf{I}_{N}) \right) \mathbf{D}^{A}_{k}(\mathbf{I}_{D} \otimes \mathbf{Z}_{k-1}),
\end{aligned}
\end{equation}
where $\mathbf{D}^{A}_{i} := \mathrm{diag}(\mathrm{vec}(\rho'(\mathbf{Z}_{i-1}\mathbf{W}_i)))$ is a diagonal matrix of activation derivatives.
\end{corollary}

\begin{proof}
Let $\mathbf{Y}_i = \mathbf{Z}_{i-1}\mathbf{W}_i$ and $\mathbf{Z}_i = \rho(\mathbf{Y}_i)$. Since $\rho$ is element-wise, its Jacobian $\partial \mathrm{vec}(\mathbf{Z}_i) / \partial \mathrm{vec}(\mathbf{Y}_i)$ is a diagonal matrix $\mathbf{D}^A_i$. By the chain rule:
\begin{equation}
    \frac{\partial \mathrm{vec}(\mathbf{Z}_i)}{\partial \mathrm{vec}(\mathbf{Z}_{i-1})} = \frac{\partial \mathrm{vec}(\mathbf{Z}_i)}{\partial \mathrm{vec}(\mathbf{Y}_i)} \frac{\partial \mathrm{vec}(\mathbf{Y}_i)}{\partial \mathrm{vec}(\mathbf{Z}_{i-1})} = \mathbf{D}^A_i (\mathbf{W}_i^{\top} \otimes \mathbf{I}_N).
\end{equation}
Applying this recursively yields the cascaded expression for $\mathbf{G}^L_k$.
\end{proof}

\begin{corollary}
\label{corollary:J_bn/first}
For a network using both ReLU activation and Batch Normalization (BN), the Jacobians are:
\begin{equation}
\label{equ:J_bn/first}
\begin{aligned}
    \mathbf{J}_i &= (\mathbf{D}^{B}_{i} \otimes \mathbf{I}_{N})\mathbf{D}^{A}_{i}(\mathbf{W}_{i}^{\top} \otimes \mathbf{I}_{N}), \\
    \mathbf{G}_i &= (\mathbf{D}^{B}_{i} \otimes \mathbf{I}_{N})\mathbf{D}^{A}_{i}(\mathbf{I}_{D} \otimes \mathbf{Z}_{i-1}), \\
    \mathbf{G}^{L}_{k} &= \left( \prod_{i=L}^{k+1} (\mathbf{D}^{B}_{i} \otimes \mathbf{I}_{N}) \mathbf{D}^{A}_{i} (\mathbf{W}_{i}^{\top} \otimes \mathbf{I}_{N}) \right) (\mathbf{D}^{B}_{k} \otimes \mathbf{I}_{N}) \mathbf{D}^{A}_{k} (\mathbf{I}_{D} \otimes \mathbf{Z}_{k-1}),
\end{aligned}
\end{equation}
where $\mathbf{D}^{B}_{i}$ is a diagonal matrix encoding BN gradients (e.g., scaling by $\gamma/\sigma$).
\end{corollary}

\begin{proof}
BN operates uniformly across the sample dimension $N$. Its gradient with respect to the input (ignoring the mean-centering for brevity) can be represented as $(\mathbf{D}^B_i \otimes \mathbf{I}_N)$, where $\mathbf{D}^B_i$ scales each feature channel. The total Jacobian for one block (Linear + Activation + BN) is the product of these individual components in reverse order of the forward pass.
\end{proof}


\end{document}

%% file: preamble.tex
\usepackage[utf8]{inputenc} 
\usepackage[T1]{fontenc}    
\usepackage{url}            
\usepackage{booktabs}       
\usepackage{amsfonts}       
\usepackage{nicefrac}       
\usepackage{microtype}      
\usepackage{xcolor}         

\usepackage{pifont}
\usepackage{fontawesome}
\usepackage{transparent} 
\usepackage{amsmath}

\usepackage{amssymb}
\usepackage{bm}
\usepackage{adjustbox}
\usepackage{dashrule}
\usepackage{mathtools}
\usepackage{makecell}
\usepackage{colortbl}
\usepackage{graphicx}
\usepackage{multirow}

\usepackage{wrapfig}
\usepackage{float}
\usepackage{enumitem}
\usepackage{dsfont}
\usepackage{array}
\newcolumntype{x}[1]{>{\centering\let\newline\\\arraybackslash\hspace{0pt}}p{#1}}

\definecolor{blueblack}{RGB}{0, 108, 173}
\definecolor{taborange}{RGB}{235, 127, 14}
\definecolor{tabgreen}{RGB}{30, 160, 30}
\definecolor{tabpurple}{RGB}{128, 103, 189}
\definecolor{tabred}{RGB}{214, 39, 40}
\definecolor{tabpink}{RGB}{240, 20, 255}
\definecolor{tabblue}{RGB}{14, 22, 255}
\definecolor{blueviolet}{RGB}{138,43,226}
\definecolor{forestgreen}{RGB}{34,139,34}

\definecolor{sol_light_blue}{RGB}{38, 139, 210}
\definecolor{sol_blue}{RGB}{38, 139, 210}
\definecolor{nord_blue}{RGB}{38, 139, 210}
\definecolor{sol_green}{RGB}{163, 190, 140}
\definecolor{sol_red}{RGB}{220, 50, 47}
\definecolor{nord_red}{RGB}{250, 190, 192}
\definecolor{nord_green}{RGB}{182, 215, 168}
\definecolor{nord_yellow}{RGB}{255, 229, 153}

\definecolor{beer_orange}{RGB}{242, 142, 28}

\definecolor{blueblack}{RGB}{0, 108, 173}
\definecolor{taborange}{RGB}{235, 127, 14}
\definecolor{tabgreen}{RGB}{30, 160, 30}
\definecolor{tabgray}{RGB}{142, 142, 142}
\definecolor{tabpurple}{RGB}{128, 103, 189}
\definecolor{tabred}{RGB}{214, 39, 40}

\newcommand{\I}{\mathbf{I}}

\newcommand{\x}{\mathbf{x}}

\newcommand{\Q}{\mathbf{Q}}
\newcommand{\K}{\mathbf{K}}

\newcommand{\Z}{\mathbf{Z}}

\newcommand{\X}{\mathbf{X}}

\newcommand{\Y}{\mathbf{Y}}
\newcommand{\W}{\mathbf{W}}

\newcommand{\G}{\mathbf{G}}
\newcommand{\V}{\mathbf{V}}

\renewcommand{\V}{\mathbf{V}}

\newcommand{\J}{\mathbf{J}}

\RequirePackage{xspace}
\makeatletter
\DeclareRobustCommand\onedot{\futurelet\@let@token\@onedot}
\def\@onedot{\ifx\@let@token.\else.\null\fi\xspace}

\makeatother

\usepackage{tcolorbox}

\newtcolorbox{keyobservation}[1][]{
  colback=blue!5,
  colframe=blue!60!black,
  fonttitle=\bfseries,
  title=Key Observation,
  sharp corners,
  boxrule=0.8pt,
  left=1em,
  right=1em,
  top=0.8em,
  bottom=0.8em,
  #1
}

\newtcolorbox{softbox}{
  colback=blue!5,
  colframe=blue!60!black,
  boxrule=0.6pt,
  sharp corners,
  left=1em,
  right=1em,
  top=0.8em,
  bottom=0.8em
}

\newtcolorbox{notebox}{
  colback=blue!4,
  colframe=blue!25,
  boxrule=0.6pt,
  arc=4pt,
  left=1em,
  right=1em,
  top=0.8em,
  bottom=0.8em
}